\documentclass[letterpaper]{article} 

\usepackage{aaai2027}  
\usepackage[hyphens]{url}  
\usepackage{graphicx} 
\usepackage{natbib}  
\usepackage{caption} 
\usepackage{amsmath,amssymb,amsthm}

\newtheorem{theorem}{Theorem}
\newtheorem{lemma}[theorem]{Lemma}
\newtheorem{proposition}[theorem]{Proposition}

\theoremstyle{definition}
\newtheorem{definition}[theorem]{Definition}
\newtheorem{remark}[theorem]{Remark}
\newtheorem{example}[theorem]{Example}

\newcommand{\cl}[1]{\overline{#1}}
\newcommand{\fix}{\mathrm{fix}}
\newcommand{\SC}{\mathrm{SC}}
\newcommand{\conv}{\mathrm{conv}}
\newcommand{\bary}{\operatorname{bar}}
\newcommand{\law}{\operatorname{Law}}
\providecommand{\doi}[1]{doi:\,#1}

\title{Unbiased Canonical Set-Valued Oracles via Lattice Theory}
\author{Jobst Heitzig}
\affiliations{Potsdam Institute for Climate Impact Research, Potsdam, Germany\footnote{This text was produced with the help of an AI assistant, Claude Opus 5. The author takes full responsibility of the content.}}

\begin{document}
\maketitle
\begin{abstract}
An oracle that tells you how likely some future event is will often change how likely that
event then becomes, simply because you act on what it told you.
We argue that this performativity is not the thing to fix.
People consult oracles precisely in order to be informed, and hence moved, by the answer.
What we should worry about instead is that asking for a {\em self-consistent} answer, one that
still holds once it has been announced, usually leaves the oracle with several answers to pick
from, and whichever rule it uses to pick is a lever it could learn to pull.

Here, we propose to take away the choice rather than the influence.
In our approach, the oracle reports a credal set instead of a point estimate, 
which lifts the oracle's reaction function from an operator on $[0,1]$ to an isotone
operator on the 
lattice of closed credal sets.
We propose to have the oracle report that operator's
least fixed point, which exists because of the theorem of Knaster and Tarski.
As that answer is fixed by a rule laid down in advance, nothing is left to choose by the oracle.
We show the proposed solution always exists, is self-consistent, is never empty, can be computed by iterating from below, and equals the ordinary point
estimate in the unlikely case where the question is not performative after all.

For simple queries about probabilities, we propose to restrict answers to intervals and show that, under a mild monotonicity assumption, the answer is simply the interval
from the no-information baseline to the self-fulfilling equilibrium one would end
up at by iteratively querying a point oracle until the answer is self-consistent.

As our proposal is purely order-theoretic, it carries over unchanged to arbitrary random variables and
to bounded continuous statistics of their law.
Finally we show that a fixed finite family of polytopes suffices to approximate every answer
uniformly, which turns the construction into a terminating computation.

We close by placing the construction inside the Scientist AI programme, where it offers a
choice-free criterion for a step that programme currently hands to audited human judgement.
\end{abstract}

\section{Introduction}\label{sec:intro}

One recurring idea for making advanced AI safer is to build systems that only {\em predict}
and never {\em act}: non-agentic oracles, or ``Scientist AIs'', that answer questions about
the world and want nothing of their own
\citep{bengio2025scientist,fornasiere2026scientist,armstrong2018oracles,bengio2024bayesian,bengio2026honesty}.
The hope is that a system which merely reports what it believes cannot pick up the drive to
reshape the world that makes goal-directed agents dangerous.

Performativity complicates that hope.
As soon as somebody reads an oracle's output and acts on it, the output becomes one of the
causes of the very events it describes.
A predictor that is scored against what actually happens can then, at least in principle,
lower its loss by picking whichever of several self-fulfilling answers suits it.
None of this is new.
That a public prediction can make itself come true, or stop itself from coming true, was at
the heart of a long debate about whether social events can be predicted at all
\citep{morgenstern1928,merton1948,grunberg1954,simon1954,popper1957}.
It shows up again as performative prediction \citep{perdomo2020,hardt2025,perdomo2025} and, in
economics, as the fixed-point character of rational-expectations equilibria \citep{muth1961}.

\paragraph{Removing agency}
We are not trying to fix performativity.
One consults an oracle exactly because one wants to be informed, and therefore moved, by
what it says, so the answer influencing the world is the point rather than a bug.
What worries us is {\em agency}.
Asking merely that a reported value still hold once it has been announced is a weak
requirement: typically several values pass it, which is the old non-uniqueness of
self-fulfilling prophecies.
Any rule the system then uses to pick one of them is discretion, and discretion is something a
capable system could in principle bend towards outcomes it happens to like.
This is plausibly one of the routes by which hidden goals could emerge, such as a taste for
worlds that are easier to predict \citep{krueger2020}.
So we try to remove the {\em choice} rather than the influence.

\paragraph{Contribution}
Our approach for this is to ask for a credal set $C$ of estimates rather than a single estimate $p$, 
then define a suitable notion of self-consistency for such set-valued responses, 
and finally use lattice theory to select a canonical such self-consistent credal set $C^\bot$ (Section \ref{sec:lattice}).
We prove it exists, is self-consistent, is never empty, collapses to the classical point
estimate when the question is not performative, and is not vacuous.
We then show that $C^\bot$ can be reached by an iteration that stays finite at every finite stage (Section \ref{sec:iteration}).
Next we give a variant $C^\star$ that treats all self-fulfilling values even-handedly, and show
that the answer is an interval running from the baseline to an equilibrium,
computable by a simple scalar orbit (Section \ref{sec:simple}).
We then generalize our theory from simple queries of the form ``What is the probability of $B$ given $A$'' 
to queries about conditional probability distributions or arbitrary statistics of arbitrary random variables (Section
\ref{sec:beyond}).
Finally, we relate all of this to the Scientist AI programme (Section
\ref{sec:sai}).
Proofs we leave out here can be found in the technical appendix.

\paragraph{Related work}
The point-valued special case of what we call self-consistency, a public prediction that comes
true once it is announced, is exactly the correct public prediction whose existence
\citet{grunberg1954} and \citet{simon1954} established with Brouwer's theorem, in reply to the
impossibility claims of \citet{morgenstern1928} and Popper's Oedipus effect
\citep{popper1957}; see also \citet{merton1948}.
Our operator restricted to singletons gives back their reaction function, and a singleton fixed
point is precisely their condition.
The same fixed point reappears as the performatively stable point of performative prediction
\citep{perdomo2020,hardt2025}, whose efficient findability \citet{perdomo2025} recently
settled, and in economics as Muth's rational-expectations prediction \citep{muth1961}.
All of these keep a single number and ask whether it exists and how to find it.
We move to sets instead, in order to {\em absorb} the non-uniqueness rather than break it by
picking.

Using Knaster--Tarski \citep{tarski1955} in place of Brouwer trades continuity for
monotonicity, and it buys something Brouwer does not: a {\em canonical} fixed point rather than
mere existence.
The same order-theoretic machinery is what gives extremal equilibria in supermodular games
\citep{milgrom1990}.
Reporting a set connects us to imprecise probability \citep{walley1991} and, where the relevant
conditionals are only partially identified, to partial-identification bounds
\citep{manski2003,balke1994}.
That set-valued forecasts are hard to elicit with a proper scoring rule \citep{seidenfeld2012}
is what shapes the architecture in Section \ref{sec:sai}.
On the safety side, designs that keep an AI safe by letting it do nothing but answer questions
go back to \citet{armstrong2018oracles}, and the non-agentic programme we engage with is that
of \citet{bengio2025scientist,bengio2024bayesian,fornasiere2026scientist,bengio2026honesty}.

\section{Counterfactual Vs Self-Referential Queries}\label{sec:two-readings}

To stop an oracle exploiting the performative channel, \citet{bengio2024bayesian}
and \citet{fornasiere2026scientist} suggest asking only {\em counterfactual} questions.
The simplest version of this would be: how likely is
the event, assuming the answer is deleted, or has no effect on the world?

We think there is a simple problem with reading the question that way.
Once we have asked and have read the answer, the premise ``assuming I never learn this answer''
is plainly false.
We are not in that situation and never were, because we knew we would read the answer the
moment we asked.
If the outcome depends on what the querier comes to believe, the counterfactual answer
therefore describes the wrong world.

So, similar to \citet{bengio2026honesty}, we stay with the self-referential reading, what will happen once someone has read the answer,
and deal with the non-uniqueness.
Write $A$ for what we condition on, $B$ for the event we care about, and $C$ for what the
oracle reports.
For a point report $p$, self-consistency says
\begin{equation}\label{eq:pointfp}
  P(B\mid A,\,\text{answer}=p)\;=\;p ,
\end{equation}
which is just the correct public prediction of \citet{grunberg1954} and \citet{simon1954}, and
the performatively stable point of \citet{perdomo2020}.\footnote{%
    If the oracle models its own report as being caused by upstream things such as its own
    parameters, one should write $f(C):=P(B\mid A,\mathrm{do}(\text{answer}=p))$, treating the report as an
    intervention rather than as evidence, in the sense of \citet{ortega2021}.
    Nothing below depends on this, see Section \ref{sec:sai}, so we keep the simpler notation.}
Its trouble is not existence but multiplicity: \eqref{eq:pointfp} usually has several
solutions, and choosing among them is an act.

\section{Why Answer With a Set?}\label{sec:credal}

Suppose a founder asks an oracle how likely her startup is to succeed, given everything known
about it.
Her investors will read the answer too, and how hard everyone then works, and how much money
arrives, depends on what it says.
A confident answer is partly self-fulfilling; a bleak one is partly self-defeating.
Say that announcing $0.1$ makes $0.1$ come true, that announcing $0.9$ makes $0.9$ come true,
and that announcing $0.8$ also works but only precariously.
Which number should a non-agentic oracle report?

Whichever of $0.1$, $0.8$ and $0.9$ it picks, it is making a decision about the founder's world
on grounds the question never supplied.
Reporting a {\em set} instead, a credal set $C\subseteq[0,1]$ of admissible values in the sense
of imprecise probability \citep{walley1991,seidenfeld2012}, lets the oracle decline that
decision and still say something useful.
But what should self-consistency mean for a set?
A first guess is
\begin{equation}\label{eq:naive}
  P(B\mid A,C)\;\in\;C ,
\end{equation}
so that the reaction to announcing $C$ is itself one of the values $C$ admits.
This is satisfiable, but too weak in a specific and instructive way.

First, there is in general no {\em smallest} set satisfying \eqref{eq:naive}, so the obvious way
of pinning down a canonical answer does not even get off the ground.
The sets satisfying \eqref{eq:naive} are not closed under intersection.
In the startup story, $\{0.1\}$ and $\{0.9\}$ both satisfy it, since announcing either value
makes it come true, but their intersection is $\emptyset$, which cannot satisfy
\eqref{eq:naive} because $P(B\mid A,\emptyset)\in\emptyset$ is impossible.
So $\{0.1\}$ and $\{0.9\}$ are two {\em minimal} solutions with nothing below them, and picking
one of them is precisely the act of choice we are trying to design away.
Second, \eqref{eq:naive} constrains only the reaction to $C$ as a whole and says nothing about
the sub-answers $C$ tolerates, even though a recipient who narrows $C$ down and acts on the
narrower set is doing something the oracle's answer explicitly permits.
Third, at the other extreme the vacuous answer $C=[0,1]$ always satisfies it.
Asking for a genuine fixed point of a set-valued operator repairs all three at once: fixed
points of an isotone map on a complete lattice {\em do} have a smallest element, they constrain
every sub-answer, and, as we show below, they avoid $[0,1]$ unless the indeterminacy really is
that wide.

\section{The Lattice Construction}\label{sec:lattice}

Let $D=2^{[0,1]}$ be the set of {\em all} subsets of $[0,1]$, ordered by inclusion.
We put no topological or measure-theoretic restriction on answers here.
One might expect to need one, but the event we condition on is that the oracle answered $C$,
and whether that is an event at all is a question about the probability space rather than about
the shape of $C$.
We simply assume the oracle has, for each hypothetical answer, an estimate of what would then
happen, which is to say we take the map $f$ below as primitive.
Restricting the answers becomes useful later, for a quite different and purely order-theoretic
reason, and we introduce it in Section \ref{sec:iteration} exactly where it is first needed.

\begin{lemma}\label{lem:Dcomplete}
$D$ is a complete lattice.
The infimum of a family $\{C_i\}$ is $\bigcap_i C_i$, the supremum is $\bigcup_i C_i$, the least
element is $\emptyset$ and the greatest is $[0,1]$.
\end{lemma}

Write $f(C):=P(B\mid A,C)$ for what the oracle thinks the probability of $B$ is, given $A$ and
given that it chooses to answer $C$ to this very question (treated by the oracle as an intervention, 
rather than modeled as a consequence of its training).
We treat $f$ as given, defined on every candidate answer $C\subseteq[0,1]$.
In particular $f(\emptyset)$ is defined: it is the reaction to the degenerate output
``$\emptyset$'', which one may read as refusing to answer, and hence as a no-information
baseline.

What we really want from an answer is more than \eqref{eq:naive}.
A recipient who is told $C$ is entitled to narrow it down: they may adopt any $C'\subseteq C$ as
their own credal set and then act as though that had been the answer, and the oracle's report
gives them no reason not to.
So the report should still be right whichever $C'$ they settle on, and it should not promise
values that no such $C'$ could produce.

\begin{definition}[Strong self-consistency]\label{def:strong}
A credal set $C$ is {\em strongly self-consistent} if it is precisely the set of probabilities
that can arise from a recipient adopting some sub-answer of $C$ and acting on it:
\[
  C \;=\; \bigl\{\,f(C')\;:\;C'\subseteq C\,\bigr\} .
\]
\end{definition}

Read the two inclusions separately.
From left to right, every value the answer admits is one that some faithful reading of the
answer actually brings about, so nothing is promised idly.
From right to left, every reading of the answer lands back inside the answer, so no way of
using the report can falsify it.
Now, \eqref{eq:naive} has the form $f(C)\in C$ rather than $C\in f(C)$, so fixed-point theorems
for correspondences such as Kakutani's do not apply. 
topology we control.
What we do have is a lattice order, which is what the Knaster--Tarski theorem runs on.
To exploit the latter, we need an isotone self-map of $D$ whose fixed points solve our
problem.
Write $G(C):=\{f(C)\}$ for the smallest answer containing the reaction to $C$, and take the
supremum of these over all sub-answers:
\begin{equation}\label{eq:F}
  F(C)=\textstyle\sup_{C'\subseteq C}G(C')=\bigl\{f(C'):C'\subseteq C\bigr\}
\end{equation}
This is isotone, since enlarging $C$ enlarges the family of sub-answers the supremum runs over,
and it dominates $G$, since $C'=C$ is one of them; it is also the smallest such map, as any
operator whose value at $C$ contains $f(C)$ dominates $G$ pointwise.
In words: collect what the world would do in response to {\em every} sub-answer that $C$
tolerates, not just to $C$ itself.
Comparing with Definition \ref{def:strong}, the equation $F(C)=C$ {\em is} strong
self-consistency, so the fixed points of $F$ are exactly the answers we were looking for.

\begin{theorem}[Existence]\label{thm:exist}
$F$ is an isotone self-map of $D$.
So by Knaster--Tarski \citep{tarski1955} its fixed points $\fix(F)$ form a nonempty complete
lattice, with a least element $C^\bot=\inf\{C:F(C)\subseteq C\}$ and a greatest element $C^\top$.
\end{theorem}

\begin{theorem}[Self-consistency]\label{thm:consistent}
$\fix(F)$ is exactly the set of strongly self-consistent credal sets.
In particular every $C\in\fix(F)$ satisfies $f(C')\in C$ for every $C'\subseteq C$, and hence
$f(C)\in C$.
Every $C\in\fix(F)$ also contains $f(\emptyset)$, so no fixed point is empty and in particular
$C^\bot\neq\emptyset$.
\end{theorem}

The second half of the theorem is what saves us from the degenerate answer.
Since $\emptyset$ is a subset of everything, its reaction gets forced into every fixed point,
so $\emptyset$ is not itself a fixed point and the least one has something in it.
The first half is the strengthening of \eqref{eq:naive} we were after.
It says that a recipient may refine the answer however they like, adopting any sub-answer
$C'\subseteq C$ and acting as though that had been the report, and the report is still correct:
what then happens is still one of the values $C$ admits.

\begin{proposition}[Collapse]\label{prop:collapse}
If the question is not performative, i.e.\ $f$ is constant, then
$C^\bot$ is the singleton $\{P(B\mid A)\}$.
\end{proposition}

\begin{proposition}[Nontriviality]\label{prop:nontrivial}
If there is any fixed point $C\subsetneq[0,1]$, then $C^\bot\subseteq C\subsetneq[0,1]$.
\end{proposition}

Read the second one backwards: $C^\bot$ can only be all of $[0,1]$ if {\em every} fixed point
is, that is, only when the indeterminacy really does span the whole interval.
The vacuous answer is thus not an artefact of our construction.

\paragraph{The question we propose asking}
Putting this together, we would have the oracle answer the following.
\begin{center}
\fbox{\begin{minipage}{0.93\columnwidth}\itshape
Report the smallest $C\subseteq[0,1]$ with
\[\{P(B\mid A,C'):C'\subseteq C\}=C,\]
where $P(B\mid A,C')$ is your estimate of the probability of $B$, given $A$ and given that you
answered $C'$ to this very question.
\end{minipage}}
\end{center}

\paragraph{Three things ``unbiased'' could mean}
These come apart, so it is worth keeping them separate.
The rule is {\em agency-neutral}: the answer follows from the inclusion order alone, and the
system has nothing to select.
This is the property our safety motivation actually needs.
It is also {\em value-symmetric}: ``least fixed point'' only ever mentions $\subseteq$ on
subsets of $[0,1]$ and never the order of $[0,1]$ itself, so relabelling $p\mapsto 1-p$ changes
nothing, and neither small nor large probabilities are favoured.
It is, however, {\em not} guaranteed to be {\em even-handed}, meaning that it need not contain
every self-consistent point estimate.
That gap is the reason for the variant in Section \ref{sec:simple}.

\begin{remark}\label{rem:image} Why not just report the image of the reaction?
There is a tempting choice-free answer: report $J:=\{f(C):C\in D\}$, everything the world
might do.
It is canonical, and weakly self-consistent, since $f(J)$ is itself a reaction and so lies in
$J$.
But every fixed point satisfies $C=F(C)\subseteq J$, so $J\supseteq C^\top\supseteq C^\bot$: it
sits above the entire fixed-point lattice and tells us almost nothing.
Worse, while $F(J)\subseteq J$, the other inclusion generally fails, so announcing $J$ does not
reproduce $J$, which is exactly the defect of an answer that stops being true once it is read.
$J$ is thus the free-but-useless baseline that motivates both of our decisions: to ask for the
genuine equation $F(C)=C$ rather than merely $f(C)\in C$, and then to take the {\em least} such
$C$.
\end{remark}

\section{Computing the Answer}\label{sec:iteration}

Theorem \ref{thm:exist} only tells us $C^\bot$ is there.
It can also be reached from below.
Put $C_0=\emptyset$ and $C_{i+1}=F(C_i)$, an increasing chain since $\emptyset\subseteq
F(\emptyset)$ and $F$ is isotone.

\begin{lemma}\label{lem:finite}
Every $C_i$ with $i<\omega$ is finite, with $|C_{i+1}|\le 2^{|C_i|}$.
\end{lemma}

So the iteration is thoroughly elementary: at each stage we apply $f$ to each of the finitely
many subsets of $C_i$ and collect what comes back, starting from $C_1=\{f(\emptyset)\}$.
One might now hope that the union $\bigcup_{i<\omega}C_i$ of the chain is already the answer.
It is not, and the reason is worth seeing, because it is what makes us restrict the lattice.

\begin{example}[Union fails to be a fixed point]\label{ex:union}
Let $g$ be continuous and increasing with $a<g(a)$, and let the oracle be believed to react to
the largest value it admits, $f(C)=g(\sup C)$ for $C\neq\emptyset$ and $f(\emptyset)=a$.
The stages are then $C_{i+1}=\{a,g(a),\dots,g^i(a)\}$, so $\bigcup_{i<\omega}C_i$ is the bare
orbit $\{g^k(a):k\ge0\}$, which does not contain its own limit $p^\star=\lim_k g^k(a)$.
But that orbit is one of its own sub-answers, so applying $F$ to it yields
$f(\{g^k(a):k\ge0\})=g(p^\star)=p^\star$, which is not in the union.
Hence $F(\bigcup_i C_i)\neq\bigcup_i C_i$.
\end{example}

Notice that this $f$ is perfectly well behaved: it is continuous for the Hausdorff metric,
since $\sup$ is.
So continuity is not the missing ingredient.
What goes wrong is that a union need not contain the limits its members are heading for.
The answer omits precisely the equilibrium that a querier hunting for a self-consistent point estimate by performing fixed-point iteration on $f$ would arrive at, which
is exactly the kind of answer that stops being true once it is read.

\paragraph{Restricting to closed answers}
The cure is to make the lattice close up under limits.
Let $D_{\mathrm{cl}}\subseteq D$ be the {\em closed} subsets of $[0,1]$, ordered by inclusion,
and let
\begin{equation}\label{eq:Fcl}
  F_{\mathrm{cl}}(C)=\cl{\{\,f(C'):C'\subseteq C\,\}}=\cl{F(C)} .
\end{equation}

\begin{lemma}\label{lem:Dcl}
$D_{\mathrm{cl}}$ is a complete lattice, with $\bigcap_i C_i$ as infimum and
$\cl{\bigcup_i C_i}$ as supremum, and $F_{\mathrm{cl}}$ is an isotone self-map of it.
Hence $F_{\mathrm{cl}}$ has a least fixed point $C^\bot_{\mathrm{cl}}$, and
$C^\bot\subseteq C^\bot_{\mathrm{cl}}$.
\end{lemma}

\paragraph{One pair of operators per answer lattice}
This is the first of several restrictions, so we fix the pattern once: every answer lattice $E$
below carries its own version of $G$ and \eqref{eq:F},
\begin{equation}\label{eq:FE}
\begin{gathered}
  G_E(C)=\textstyle\bigvee_E\bigl\{f(C)\bigr\},\\
  F_E(C)=\textstyle\bigvee_E\bigl\{\,G_E(C'):C'\subseteq C\,\bigr\},
\end{gathered}
\end{equation}
with joins taken in $E$, and $C^\bot_E$ denotes the least fixed point of $F_E$; associativity of
joins makes $F_E(C)$ equally the join of the reactions themselves.
For the powerset the join is the union, recovering $G(C)=\{f(C)\}$ and \eqref{eq:F}; for closed
sets it is the closure of the union, giving $F_{\mathrm{cl}}$, with $G_{\mathrm{cl}}=G$ since
singletons are already closed.
Later come closed intervals ($F_{\mathrm{int}}$, Section \ref{sec:simple}), closed convex sets
($F_{\mathrm{cvx}}$, Section \ref{sec:beyond}) and a finite family of polytopes ($F_k$, Section
\ref{sec:finite}); all are isotone for the same reason, and what changes is only how far the
join rounds the answer outwards.

The two lattices agree on everything finite, so Lemma \ref{lem:finite} is untouched: the stages
$C_i$ are finite, hence closed, and the closure in \eqref{eq:Fcl} does nothing below $\omega$: all of the topology sits in the
single step $C_\omega=\cl{\bigcup_{i<\omega}C_i}$.
What we give up is a little of Definition \ref{def:strong}: a fixed point of $F_{\mathrm{cl}}$
still cannot be falsified by any reading of it, since $f(C')\in C$ for every $C'\subseteq C$,
but the converse now holds only up to limits, in that a value the answer admits may be
approached by readings of it rather than produced by one exactly.

\begin{theorem}[Reachability in the limit]\label{thm:kleene}
If $f$ is continuous for the Hausdorff metric, then $C_\omega=C^\bot_{\mathrm{cl}}$.
\end{theorem}

Without continuity the chain still climbs, and one simply keeps going transfinitely, with
$C_\lambda=\cl{\bigcup_{\alpha<\lambda}C_\alpha}$ at limits, until it stabilises at some
closure ordinal.
Continuity is precisely what collapses that ordinal back to $\omega$.

Allowing only closed sets is not the last restriction we shall make: closed sets buy termination but are still
unwieldy to compute with. In Sections \ref{sec:simple} and \ref{sec:beyond} we cut the lattice down to closed
{\em intervals} or {\em convex} sets, and in Section \ref{sec:finite} to a finite family, so that the iteration stops
after boundedly many steps.

\section{Variants for Simple Queries}\label{sec:simple}

\paragraph{Being even-handed}
$C^\bot$ need not contain all self-fulfilling point answers.
It might stay within the basin of attraction containing the baseline $a:=f(\emptyset)$ and never find the self-fulfilling values elsewhere.
Let $\SC:=\{p:f(\{p\})=p\}$ be the self-consistent singletons and $x:=\{a\}\cup\SC$.

\begin{proposition}\label{prop:cstar}
$x\subseteq F(x)$.
Hence the order interval $[x,[0,1]]$ is $F$-invariant and carries a least fixed point
$C^\star$, the smallest strongly self-consistent credal set that contains the baseline and
every self-consistent singleton.
Moreover $C^\bot\subseteq C^\star$.
On the closed lattice the same construction works with $x$ replaced by $\cl{x}$, except that
$C^\star_{\mathrm{cl}}$ is then strongly self-consistent only up to closure, in the sense of
Section \ref{sec:iteration}.
\end{proposition}

Whether to report $C^\bot$ or $C^\star$ is a real modelling decision, between saying as little as
possible and treating all equilibria alike.
Both are fixed by a rule, so both are agency-neutral.

\paragraph{A worked example}
To see that this distinction really bites, let us make the startup story concrete.
Let the singleton reaction be $g(p)=f(\{p\})=p+h(p)$ with
$h(p)=(p-\tfrac1{10})(p-\tfrac8{10})(p-\tfrac9{10})(\alpha p+\beta)$,
$\alpha=-\tfrac{25}{12}$ and $\beta=-\tfrac{25}{36}$.
The three explicit roots put the diagonal crossings at $0.1$, $0.8$ and $0.9$, and the linear
factor makes $g$ increasing on $[0,1]$ with range $[0.05,0.95]$.
Numerically $g'(0.1)\approx0.49$, $g'(0.8)\approx1.17$ and $g'(0.9)\approx0.79$, so
$\SC=\{0.1,0.8,0.9\}$, with $0.1$ and $0.9$ stable and $0.8$ the unstable threshold between
their basins.
For sets, let the founder react to the range of $C$ but be unsettled by ambiguity, which drags
the outcome towards a neutral baseline: with midpoint $m$ and width $w$ of
$[\min C,\max C]$, put
\[
  f(C)=(1-\lambda w)\,g(m)+\lambda w\,b,\qquad b=0.4,\ \lambda=0.6,
\]
and $f(\emptyset)=b$, so the anchor is $a=0.4$.
Singletons give $f(\{p\})=g(p)$ as they should, and $f([0.1,0.9])\approx0.41$: told the whole
span of equilibria, the founder ends up somewhere unremarkable near the baseline.

Discretising $[0,1]$ and iterating $F$ up from $\emptyset$---on a finite grid every subset is
finite, hence closed, so $F$ and $F_{\mathrm{cl}}$ coincide---gives $C^\bot=[0.10,\,c]\cup T$ with
$c\approx0.26$, where $T$ is a thin scatter of isolated points along the descending transient
$0.40, 0.31, 0.27,\dots$
So $C^\bot$ is a filled lower basin with a discrete tail attached, and not an interval; the gaps
multiply rather than disappear as we refine the grid, so they are real.
It contains $0.1$ but neither $0.8$ nor $0.9$.
The reason is simply that the baseline $0.4$ lies below the unstable threshold, so the orbit
starting from the empty answer drains into the lower basin and never discovers the optimistic
equilibria.
Iterating instead from $\{a\}\cup\SC$ gives the clean interval $C^\star=[0.10,\,0.90]$,
the full span of the self-consistent singletons, with $C^\bot\subseteq C^\star$ as guaranteed.
The example shows both what is good and what is awkward about the least-fixed-point rule.
It is minimal, but it can come out fractured, and it can be captured by whichever basin the
baseline happens to fall into.
$C^\star$ fixes the latter, at the price of a wider answer.
The reason for the fracturing is that this $f$ violates the no-overshoot condition
\eqref{eq:c1} below: the ambiguity pull drags the reaction outside the local range of $g$ for
narrow announcements, and that overshoot is exactly what grows the transient tail.

\paragraph{Restricting to intervals: from baseline to equilibrium}
The worked example gave a fractured $C^\bot$, which is awkward to report and awkward to compute
with, so let us simply {\em require} the answer to be an interval.
The closed subintervals of $[0,1]$ form a complete lattice, with intersection as meet and convex
hull of the union as join; write $F_{\mathrm{int}}$ for the resulting operator \eqref{eq:FE} and
$C^\bot_{\mathrm{int}}$ for its least fixed point.
On this much smaller lattice the whole construction collapses to iterating a single real-valued
map, and the result has a reading in plain words.
Write $g(p):=f(\{p\})$ for the singleton reaction, and consider the following no-overshoot
condition:
\begin{equation}\label{eq:c1}
\begin{gathered}
  f(C) \in \textstyle \left[\inf_{p\in C}g(p),  \sup_{p\in C}g(p)\right]
  \quad\forall C\neq\emptyset.
\end{gathered}\tag{C1}
\end{equation}
I.e., announcing a spread of values never provokes a reaction outside the range of
reactions to the individual values in it.

\begin{theorem}[The answer is an anchor-to-equilibrium interval]\label{thm:interval}
Let $g$ be continuous and nondecreasing and let \eqref{eq:c1} hold.
Let $a=f(\emptyset)$ and let $p^\star=\lim_{k\to\infty}g^k(a)$, the stable fixed point of $g$ in
whose basin the baseline lies.
Then
\[
  C^\bot_{\mathrm{int}}=\bigl[\min(a, p^\star), \max(a,p^\star)\bigr].
\]
\end{theorem}

So under these hypotheses the answer reads off very simply: it runs from the no-information
baseline to the self-fulfilling equilibrium that a querier who starts there will actually end
up at.
It is computed by the scalar orbit $x_0=a$, $x_{k+1}=g(x_k)$, which one may stop as soon as
$x_{k+1}\in\conv\{x_0,\dots,x_k\}$; no set ever has to be represented explicitly.
Leaving out the equilibria in other basins is right rather than an oversight, since a querier
honestly starting at $a$ will not reach them.
That is also why this is $C^\bot_{\mathrm{int}}$ and not $C^\star$: a monotone orbit cannot
cross the unstable fixed point that separates the basins.
Note that the theorem is about the interval lattice.
On the full lattice of closed sets the same hypotheses give only the orbit together with its
limit, a discrete set; it is the convex hull in the join that fills the gaps.
Note also that the worked example above satisfies neither hypothesis of the theorem in full:
its ambiguity pull overshoots the local range of $g$ for narrow announcements, violating
\eqref{eq:c1}, and that overshoot is exactly what grows the transient tail.

One might instead want to report the interval up to the {\em nearest} stable fixed point.
That sounds better but is wrong.
It need not be self-consistent, since announcing it can push the querier into the other
equilibrium; it names an endpoint the dynamics never reach; its discontinuity sits at the
midpoint between the two equilibria, which means nothing dynamically, instead of at the genuine
bifurcation; and choosing it optimises some external objective, which is exactly the discretion
we set out to remove.

\section{Beyond Simple Queries}\label{sec:beyond}

Nothing so far used the structure of $[0,1]$ beyond its being a compact metric space, and
nothing used the answer being the probability of an event. So we can generalize:

\paragraph{Arbitrary random variables}
Replace $B$ by an arbitrary random variable $X$ and $P(B\mid A,C)$ by the conditional law
$\law(X\mid A,C)$, so that credal sets are now closed subsets of the space
$\mathcal P(\mathcal X)$ of laws.
Lemma \ref{lem:Dcomplete}, Theorems \ref{thm:exist} and \ref{thm:consistent}, and Propositions
\ref{prop:collapse} and \ref{prop:cstar} all go through word for word.

\paragraph{Convex credal sets}
It is natural to restrict further to closed {\em convex} sets, a credal set being convex by
definition anyway.
These again form a complete lattice, with intersection as meet and closed convex hull as join,
so Knaster--Tarski still applies, and convexity sweeps away the isolated points the full
closed-set lattice can produce.
Here \eqref{eq:FE} reads $F_{\mathrm{cvx}}(C)=\cl{\conv}\{f(C'):C'\subseteq C\}$, with least
fixed point $C^\bot_{\mathrm{cvx}}$.

Things become explicit in the affine case.
If $f(C)$ is an affine function of the barycenter of $C$, $f(C)=T(\bary\,C)$, then the barycentres of subsets fill $C$, so
$F_{\mathrm{cvx}}(C)=\cl{\conv}(T(C)\cup\{a\})$ and
\begin{equation}\label{eq:orbithull}
  C^\bot_{\mathrm{cvx}}=\cl{\conv}\{\,T^n a: n\ge 0\,\},
\end{equation}
the closed convex hull of the anchor's orbit.
Its shape is decided by the spectrum of $T$'s linear part (Figure \ref{fig:simplex}).
A complex pair makes the orbit spiral into the stationary law, a negative eigenvalue makes it
zig-zag across it, and a positive real spectrum makes it creep in monotonically along a thin
lens, which degenerates to the segment from baseline to equilibrium and so is the direct
analogue of the binary interval.
If the orbit ever lands inside the convex hull of its predecessors it stays there, by
affineness, and $C^\bot_{\mathrm{cvx}}$ is then a polytope; a purely positive real spectrum never encircles the
limit, so in that case $C^\bot_{\mathrm{cvx}}$ has infinitely many extreme points.

\begin{figure}[t]\centering
\includegraphics[width=\linewidth]{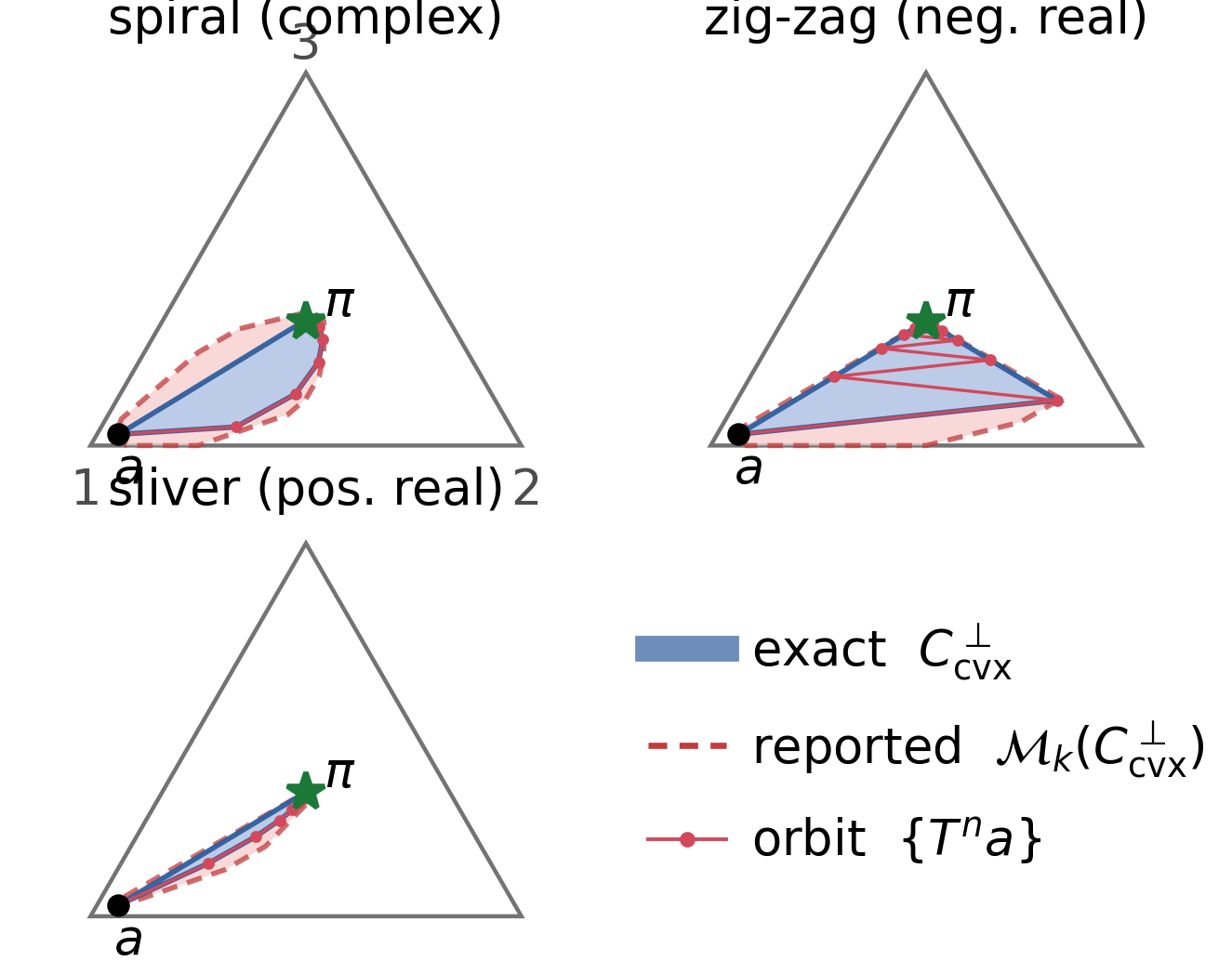}
\caption{A three-valued $X$ with an affine reaction; simplex vertices $X=1,2,3$ are marked in the
first panel. $C^\bot_{\mathrm{cvx}}$ is the closed convex hull of the anchor's orbit. Its shape is decided
by the spectrum of $T$'s linear part: $1,\,0.52\pm0.25i$ spirals in, $1,-0.70,0.70$ zig-zags
across the stationary law $\pi$, and $1,\,0.63,\,0.52$ creeps in along a thin lens. Dashed red is
the finite-family hull of Section \ref{sec:finite} we would report, here at a deliberately coarse $k=4$, so the two can
be told apart.}
\label{fig:simplex}
\end{figure}

\paragraph{Arbitrary statistics}
Most generally, let $Y:\mathcal P(\mathcal X)\to\mathbb R^d$ be some statistic of the law of $X$---a
mean, a covariance, a vector of moments or quantiles---, where $\mathcal P(\mathcal X)$ is the space of probability measures on the value set of $X$, and have the oracle report a credal set of values of $Y$, reacting by $f(C)=Y(\law(X\mid A,C))$.
Existence needs nothing at all here: each of these lattices is complete and its operator
\eqref{eq:FE} isotone by construction, which is all Knaster--Tarski consumes, so no assumption on
$Y$ is used.

\begin{proposition}\label{prop:general}
If $C\mapsto\law(X\mid A,C)$ is continuous and $Y$ is bounded and continuous, then on
both \textup{(a)} all closed and \textup{(b)} all closed convex subsets of
$K:=\cl{Y(\mathcal P(\mathcal X))}$, the operators $F_{\mathrm{cl}}$ and $F_{\mathrm{cvx}}$ are
$\omega$-Scott-continuous and their least fixed points are reached at stage $\omega$.
\end{proposition}

Boundedness makes $K$ compact, continuity makes $f$ continuous, and the approximation argument
survives because closed convex sets are closed under intersection.
If $Y$ is unbounded we adjoin the ambient space as a top element.
Existence survives that, but chains need no longer converge, and then, just as when $Y$ is
merely discontinuous, reaching the least fixed point may take transfinite iteration.

\section{Finite Approximation}\label{sec:finite}

Everything so far has been exact and infinite.
Even the interval lattice of Section \ref{sec:simple} has continuum many members, and
Theorem \ref{thm:kleene} only promises convergence after $\omega$ steps.
For an oracle that has to return an answer, we want a lattice that is finite, so that the
iteration stops, while still being fine enough that the answer is not thrown away.

\paragraph{The motivating case}
For $d=1$ the obvious candidate is fixed-precision intervals: fix $n$ and let
$\mathcal M^{(n)}$ consist of $\emptyset$ together with all $[\ell,u]\subseteq[0,1]$ whose
endpoints lie in $\tfrac1n\mathbb Z$.
This family is finite, contains $[0,1]$, and is closed under intersection, so it induces a
closure operator $\mathcal M^{(n)}(C)$, namely rounding the endpoints of $C$ outwards; and
$d_H(C,\mathcal M^{(n)}(C))\le 1/n$ for every interval $C$.
Reporting to two decimal places is exactly this.

\paragraph{The general case}
We want the same in higher dimensions, where answers are convex subsets of a cube
$[0,1]^d$, and the difficulty is that the obvious generalisation fails.
Polytopes with vertices on a grid are {\em not} closed under intersection, since two
grid-vertex simplices generally meet in a polytope with off-grid vertices.
One could close that family under intersections, but our solution is simpler.
We fix the facet normals and offsets instead of the vertices.
We believe neither the construction nor the resulting bound is new, but were unable to find the exact statements in the literature, so we state our own version here. 
Let
\[
  Z_k=\{-k,\dots,k\}^d\setminus\{0\},\qquad \Gamma_k=\tfrac1k\mathbb Z\cap[-dk,dk],
\]
let $H_{z,c}=\{x\in[0,1]^d:\langle z,x\rangle\le c\}$ for $z\in Z_k$ and $c\in \Gamma_k$, and let
$\mathcal M_k$ be the set of all intersections of subfamilies of these.
Fixing the normals in advance is the {\em template polyhedra} idea of
\citet{sankaranarayanan2005,sankaranarayanan2008}; discretising the offsets to a grid is what
makes the family finite.

\begin{lemma}\label{lem:moore}
$\mathcal M_k$ is a finite family of polytopes containing $\emptyset$ and $[0,1]^d$ and closed
under intersection, hence a complete lattice under inclusion, with intersection as meet.
Its closure operator is computed from support functions,
\[
  \mathcal M_k(C)=\textstyle\bigcap_{z\in Z_k}H_{z,c_z},~
  c_z=\min\{c\in \Gamma_k: c\ge h_C(z)\},
\]
where $h_C(z)=\sup_{y\in C}\langle z,y\rangle$, so $\mathcal M_k$ itself is never enumerated.
For $d=1$ it contains every fixed-precision interval family $\mathcal M^{(n)}$ with $n\le k$.
\end{lemma}

\begin{theorem}[Uniform approximation]\label{thm:approx}
For every convex $C\subseteq[0,1]^d$,
$d_H\bigl(C,\mathcal M_k(C)\bigr)\le 3d/k$.
Hence for $k\ge 3d/\varepsilon$ one fixed finite family is within $\varepsilon$ of {\em every}
convex answer.
\end{theorem}


\paragraph{The operator}
The lattice approximated here is the closed convex one of Section \ref{sec:beyond}, whose members
are convex as Theorem \ref{thm:approx} requires, so the finite construction tracks
$F_{\mathrm{cvx}}$ and $C^\bot_{\mathrm{cvx}}$.
Write $G_k(C):=\mathcal M_k(\{f(C)\})$ for the smallest member of the family containing the
single reaction to $C$: this is $G_E$ of \eqref{eq:FE} for $E=\mathcal M_k$, standing to $G$ as
$F_k$ stands to $F$.
By Lemma \ref{lem:moore} and $h_{\{y\}}(z)=\langle z,y\rangle$ this is explicit:
\begin{equation}\label{eq:Gk}
  G_k(C)=\textstyle\bigcap_{z\in Z_k}\left\{x
  :\langle z,x\rangle\le
        \tfrac1k\bigl\lceil k\,\langle z,f(C)\rangle\bigr\rceil\right\},
\end{equation}
one rounding per direction.
The finite operator is then the join of these over all sub-answers,
\begin{equation}\label{eq:Fk}
  F_k(C):=\textstyle\mathcal M_k\bigl(F_{\mathrm{cvx}}(C)\bigr)=\bigvee_{C'\subseteq C}G_k(C') ,
\end{equation}
which is \eqref{eq:FE} for $E=\mathcal M_k$, hence isotone, and satisfies
$F_{\mathrm{cvx}}(C)\subseteq F_k(C)$.
Using $F(C)$ in place of $F_{\mathrm{cvx}}(C)$ would give the same, since members of
$\mathcal M_k$ are closed and convex.

\begin{theorem}[Finite computation]\label{thm:finite}
$F_k$ has a least fixed point $C^\bot_k$ in $\mathcal M_k$, and a least fixed point $C^\star_k$
above $\mathcal M_k(\{a\}\cup\SC)$.
Both are reached by iteration from below, as in Lemma \ref{lem:finite}, in at most
$|Z_k|\,(|\Gamma_k|-1)+1$ steps, where $|Z_k|=(2k+1)^d-1$ and $|\Gamma_k|=2dk^2+1$.
Moreover $C^\bot_{\mathrm{cvx}}\subseteq C^\bot_k$ and
$C^\star_{\mathrm{cvx}}\subseteq C^\star_k$.
\end{theorem}

The last clause provides safety: the finite answer is an {\em outer} approximation, so it never
claims more precision than the exact construction warrants. What it gives up is exactness of
self-consistency, and the right way to say how much is:

\begin{definition}[Approximate self-consistency]\label{def:approx}
A convex credal set $C$ is {\em $\varepsilon$-self-consistent} if
$d_H\bigl(C,F_{\mathrm{cvx}}(C)\bigr)\le\varepsilon$, that is, if announcing $C$ reproduces $C$
to within $\varepsilon$.
\end{definition}

\begin{proposition}\label{prop:approxsc}
Every fixed point of $F_k$ is $(3d/k)$-self-consistent.
In particular, for $k\ge 3d/\varepsilon$, the computed answers $C^\bot_k$ and $C^\star_k$ are
$\varepsilon$-self-consistent.
\end{proposition}

So the precision of the answer and its self-consistency are the same,
and both are set by the single parameter $k$.

\section{Relation to the Scientist AI Programme}\label{sec:sai}

We see this work as complementary to the non-agentic Scientist AI programme rather than as
competing with it, and it meets that programme at one specific point.

\paragraph{Predictor and Scaffold}
\citet{bengio2026honesty} separate a Predictor, trained towards a Bayesian posterior with no
consequence-dependent signal, from a Scaffold that decides what actually gets deployed.
That split matches how one would implement our proposal rather closely.
An internal network $h_\theta$ that predicts $P(B\mid A,C)$ or $Y(\law(X\mid A,C))$ for hypothetical reports $C$ plays
the Predictor, and the fixed-point search that iterates $h_\theta$ up to the canonical answer
plays the selecting half of the Scaffold.
Because the report is chosen by the search and not by $\theta$, the network is graded on
accuracy alone, as long as no gradient runs back through which fixed point the search settled
on.
And since $h_\theta$ predicts point probabilities even though a set gets reported,
cross-entropy against the realised outcome is still a proper scoring rule, which neatly
sidesteps the difficulty of eliciting set-valued forecasts directly \citep{seidenfeld2012}.

\paragraph{Interventions}
That programme trains its Predictor to answer queries in which its own deployed output is set by
intervention rather than observed as evidence, and blocks the deployment channel from carrying
any learning signal back into the model---mechanically the counterfactual teaching of
\citet{ortega2021}, introduced there to stop a sequence model reading its own outputs as
evidence about the world.
The same discipline applies to us.
Writing $f(C)=P(B\mid A,\mathrm{do}(C))$ makes our fixed points fixed points of the causal map.
The distinction only bites when the predictor models whatever process generates its own report;
if hypothetical reports are simply handed to it from outside, conditioning and intervening
coincide.

\paragraph{A tie-break with no discretion in it}
\citet{bengio2026honesty} ask that any selection made outside the Predictor, whether gating what
gets deployed, choosing among internally coherent candidate outputs, or breaking the tie between
several performatively consistent fixed points, follow criteria written down in advance and open
to audit.
Among the criteria they offer are calibration measures, residuals of the
performative-consistency constraint, and a preference for whichever candidate carries the lower
predicted probability of harm.

Our construction supplies a candidate criterion for exactly that step, and an unusually clean
one.
$C_E^\bot$ has no parameters beyond the choice of lattice $E$, is value-symmetric and deterministic, hence about as auditable as a
rule can be, and carries no agency. Two caveats.
Because that design treats deployed predictions as interventions to get a unique posterior with
no fixed point needed at all, the multiplicity we address turns up at their Scaffold, not their
Predictor.
And a Scaffold deploys one output whereas we report a set. 

\paragraph{Where else a set would help}
The same move looks useful elsewhere in that design.
A guardrail comparing a point estimate of harm against a threshold becomes, under a credal harm
estimate, a three-way decision---deploy, abstain, or defer---by gating on the upper probability,
which turns a stated preference for caution under ambiguity into a mechanism.
Sensitivity to the prior, and to how a safety specification is read, is the classical home
ground of robust-Bayes credal sets, as is partial identification of the relevant conditionals
\citep{manski2003,balke1994}.

\section{Open Questions and Conclusion}\label{sec:open}

Plenty is still open.
Can we say anything general about the shape of $C^\bot$, and when does it equal $C^\star$?
Which selection principle, minimality or even-handedness, is the better notion of `canonical'?
How should $f$ be estimated away from the realised distribution, given that only the report
actually made ever gets labelled by an outcome?

The natural self-consistent objects line up as
\[
 \emptyset \subset \{f(\emptyset)\} \subseteq C_E^\bot \subseteq C_E^\star \subseteq C_E^\top
 \subseteq \cl{\mathrm{Im}\,f},
\]
from the forced anchor through the least-committal answer and the even-handed variant, up to the
greatest fixed point and the vacuous reaction image of Remark \ref{rem:image}.
The candidates worth reporting sit near the bottom, and Section \ref{sec:finite} makes them
computable to within $\varepsilon$ in boundedly many steps.
Reporting one of them, by a rule fixed in advance that looks only at the lattice order, lets an
oracle take on the self-referential questions a counterfactual design must set aside, without
the agency such designs exist to exclude.

\appendix
\section{Technical Appendix}
\noindent
This appendix collects the proofs left out of the main text, together with a few remarks that
did not fit there.
Results restated from the main text point back to it; results that only appear here are numbered
on their own.
Throughout, $D=2^{[0,1]}$ is the powerset of $[0,1]$ ordered by inclusion, $f$ is the reaction
map defined on every subset, $a:=f(\emptyset)$ is the anchor, and
\[
  F(C)=\{\,f(C'):C'\subseteq C\,\} .
\]
Each further answer lattice $E$ below carries its own pair
$G_E(C)=\bigvee_E\{f(C)\}$ and $F_E(C)=\bigvee_E\{G_E(C'):C'\subseteq C\}$, the joins taken in
$E$, with least fixed point $C^\bot_E$ of $F_E$: the closed subsets $D_{\mathrm{cl}}$ with $F_{\mathrm{cl}}(C)=\cl{F(C)}$ from
Section \ref{app:iteration}, the closed intervals $D'$ with $F_{\mathrm{int}}$ from Section
\ref{app:interval}, the closed convex sets with $F_{\mathrm{cvx}}(C)=\cl{\conv}F(C)$ from
Section \ref{app:beyond}, and the finite family $\mathcal M_k$ with $F_k$ from Section
\ref{app:finite}.

\section{The Lattice and the Operator}\label{app:lattice}

\begin{lemma}[restating Lemma~\ref{lem:Dcomplete}]\label{alem:complete}
$D$ is a complete lattice.
The infimum of a family $\{C_i\}\subseteq D$ is $\bigcap_i C_i$ and the supremum is
$\bigcup_i C_i$; the least element is $\emptyset$ and the greatest is $[0,1]$.
\end{lemma}

\begin{proof}
The powerset of any set, ordered by inclusion, is a complete lattice, with arbitrary
intersections as infima and arbitrary unions as suprema.
\end{proof}

Both infima and suprema are therefore as simple as they could be, which is what lets us dispense
with any closure operator in this section.

\begin{lemma}[Isotonicity]\label{alem:isotone}
Let $G(C):=\{f(C)\}$ and $F(C):=\sup_{C'\le C}G(C')$.
More generally, for any self-map $G_0$ of a complete lattice $D$ the map
$C\mapsto\sup_{C'\le C}G_0(C')$ is isotone and dominates $G_0$.
\end{lemma}

\begin{proof}
If $C_1\le C_2$ then $\{C':C'\le C_1\}\subseteq\{C':C'\le C_2\}$, so the supremum over the
larger index set dominates, giving isotonicity.
Taking $C'=C$ in the supremum gives $G_0(C)\le\sup_{C'\le C}G_0(C')$.
\end{proof}

Taking $G_0=G$ and working out the supremum with Lemma \ref{alem:complete} gives exactly
the $F$ displayed above.
It is the smallest such choice, in the sense that any $G_0$ with $f(C)\in G_0(C)$ dominates $G$
pointwise.

\begin{theorem}[restating Theorem~\ref{thm:exist}]\label{athm:exist}
$F$ is an isotone self-map of $D$.
Hence $\fix(F)$ is a nonempty complete lattice, with least element
$C^\bot=\inf\{C:F(C)\subseteq C\}$ and greatest element $C^\top=\sup\{C:C\subseteq F(C)\}$.
\end{theorem}

\begin{proof}
$F(C)$ is a subset of $[0,1]$, so $F$ maps $D$ into $D$, and isotonicity is Lemma
\ref{alem:isotone}.
Knaster--Tarski \citep{tarski1955} applies to any isotone self-map of a complete lattice and
gives both the lattice structure of $\fix(F)$ and the two characterisations displayed, in their
pre-fixed-point form.
\end{proof}

\begin{theorem}[restating Theorem~\ref{thm:consistent}]\label{athm:consistent}
$\fix(F)$ is exactly the set of strongly self-consistent credal sets.
In particular, for $C\in\fix(F)$ we have $f(C')\in C$ for every $C'\subseteq C$, hence
$f(C)\in C$.
Moreover $a\in C$, so no fixed point is empty and $C^\bot\neq\emptyset$.
\end{theorem}

\begin{proof}
The equation $F(C)=C$ reads $\{f(C'):C'\subseteq C\}=C$, which is verbatim Definition 2 of the
paper, so the two classes coincide.
For the rest, let $C'\subseteq C$.
By definition $f(C')\in F(C)$, and $F(C)=C$, so $f(C')\in C$.
Taking $C'=C$ gives $f(C)\in C$.
Since $\emptyset\subseteq C$, the same argument with $C'=\emptyset$ gives $a\in C$.
So no fixed point is empty, and as $C^\bot$ is one of them, $a\in C^\bot$.
\end{proof}

\begin{proposition}[restating Proposition~\ref{prop:collapse}]\label{aprop:collapse}
If $f$ is constant with value $p_0=P(B\mid A)$, then $\fix(F)=\{\{p_0\}\}$ and $C^\bot=\{p_0\}$.
\end{proposition}

\begin{proof}
For any $C\in D$, the reactions to subsets of $C$ are all equal to $p_0$, including the reaction
to $C'=\emptyset$, so $F(C)=\{p_0\}$ whatever $C$ is.
A constant map has exactly one fixed point, namely its value.
\end{proof}

\begin{proposition}[restating Proposition~\ref{prop:nontrivial}]\label{aprop:nontrivial}
If there is a fixed point $C\subsetneq[0,1]$, then $C^\bot\subseteq C\subsetneq[0,1]$.
\end{proposition}

\begin{proof}
$C^\bot$ is contained in every fixed point, in particular in $C$.
\end{proof}

\begin{remark}
The contrapositive is the useful direction.
$C^\bot$ equals $[0,1]$ only if {\em every} fixed point does, which is to say only when the
reflexive indeterminacy really does run across the whole interval.
So a vacuous answer, when it happens, is telling us something about the problem rather than
about our construction.
\end{remark}

\section{Getting There by Iteration}\label{app:iteration}

Put $C_0=\emptyset$ and $C_{i+1}=F(C_i)$.
Since $\emptyset\subseteq F(\emptyset)$ and $F$ is isotone, the $C_i$ form an increasing chain.

\begin{lemma}[restating Lemma~\ref{lem:finite}]\label{alem:finite}
Every $C_i$ with $i<\omega$ is finite, with $|C_{i+1}|\le2^{|C_i|}$.
\end{lemma}

\begin{proof}
By induction.
$C_0=\emptyset$ is finite, and if $C_i$ is finite then it has $2^{|C_i|}$ subsets, so
$C_{i+1}=\{f(C'):C'\subseteq C_i\}$ has at most that many elements.
\end{proof}

Concretely $C_1=\{a\}$, and a nesting of depth $k$ such as $f(\{f(\{\dots\})\})$ appears at
stage $k+1$.
It matters here that $F$ sweeps {\em all} sub-answers at each stage rather than just applying
$f$ to the current set.
The weaker step $C_{i+1}=C_i\cup\{f(C_i)\}$ would never isolate a singleton such as $\{a\}$ and
feed it back into $f$, so its limit would generally not be a fixed point at all.

\begin{proposition}[restating Example~\ref{ex:union}]\label{aprop:union}
With $f(\emptyset)=a$ and $f(C)=g(\sup C)$ for $C\neq\emptyset$, where $g$ is continuous and
increasing with $a<g(a)$, we have $C_{i+1}=\{a,g(a),\dots,g^i(a)\}$ and
$F\bigl(\bigcup_{i<\omega}C_i\bigr)\neq\bigcup_{i<\omega}C_i$.
Moreover this $f$ is continuous for the Hausdorff metric.
\end{proposition}

\begin{proof}
The formula for $C_{i+1}$ follows by induction: the sub-answers of
$C_i=\{a,\dots,g^{i-1}(a)\}$ have suprema $g^k(a)$ for $k\le i-1$, whose images under $g$ are
$g^{k+1}(a)$, and the empty sub-answer contributes $a$.
So $U:=\bigcup_{i<\omega}C_i=\{g^k(a):k\ge0\}$, which is increasing and bounded, with limit
$p^\star:=\lim_k g^k(a)$ satisfying $g(p^\star)=p^\star$ by continuity, and $p^\star\notin U$
since the orbit is strictly increasing.
But $U\subseteq U$, so $f(U)=g(\sup U)=g(p^\star)=p^\star\in F(U)\setminus U$.
For the last claim, $\sup C=\sup\cl{C}$ and $|\sup C-\sup C''|\le d_H(C,C'')$ for nonempty
$C,C''$, so $\sup$ is $1$-Lipschitz for $d_H$ and $f=g\circ\sup$ is continuous.
\end{proof}

So the powerset lattice does not support an $\omega$-step iteration, and, as the proposition
shows, continuity of $f$ is not what is missing.
We therefore pass to the sublattice of closed answers.

\begin{lemma}[restating Lemma~\ref{lem:Dcl}]\label{alem:Dcl}
$D_{\mathrm{cl}}$ is a complete lattice, with infimum $\bigcap_i C_i$ and supremum
$\cl{\bigcup_i C_i}$, and $F_{\mathrm{cl}}=\cl{F(\cdot)}$ is an isotone self-map of it, so it
has a least fixed point $C^\bot_{\mathrm{cl}}$.
Moreover $C^\bot\subseteq C^\bot_{\mathrm{cl}}$.
\end{lemma}

\begin{proof}
Any intersection of closed sets is closed and is the greatest lower bound; any closed set
containing $\bigcup_i C_i$ contains $\cl{\bigcup_i C_i}$, which is therefore the least upper
bound, so $D_{\mathrm{cl}}$ is a complete lattice with $\emptyset$ and $[0,1]$ as extremes.
Closures are closed and $\cl{\cdot}$ is isotone, so $F_{\mathrm{cl}}$ is an isotone self-map,
and Knaster--Tarski applies.
For the last claim, $F(C^\bot_{\mathrm{cl}})\subseteq\cl{F(C^\bot_{\mathrm{cl}})}
=C^\bot_{\mathrm{cl}}$, so $C^\bot_{\mathrm{cl}}$ is a pre-fixed point of $F$ on $D$, and
$C^\bot=\inf\{C:F(C)\subseteq C\}\subseteq C^\bot_{\mathrm{cl}}$.
\end{proof}

Below, $C_\omega:=\cl{\bigcup_{i<\omega}C_i}$ denotes the supremum of the chain in
$D_{\mathrm{cl}}$.
By Lemma \ref{alem:finite} the stages are finite, hence closed, so they are the same whether
computed with $F$ or with $F_{\mathrm{cl}}$.

\begin{lemma}[Approximation]\label{alem:approx}
Write $d_H$ for the Hausdorff metric.
Then $C_i\to C_\omega$ in $d_H$, and every nonempty $C'\subseteq C_\omega$ is the $d_H$-limit of
finite sets $C'_i\subseteq C_i$.
\end{lemma}

\begin{proof}
Since $C_i\subseteq C_\omega$ we have $\sup_{x\in C_i}d(x,C_\omega)=0$.
As $\bigcup_i C_i$ is dense in the compact set $C_\omega$, finitely many of its points
$\varepsilon$-cover $C_\omega$ for each $\varepsilon>0$, and because the chain is increasing
those points all lie in a common $C_N$; so $\sup_{y\in C_\omega}d(y,C_i)\le\varepsilon$ once
$i\ge N$.
Hence $\delta_i:=d_H(C_i,C_\omega)\to0$.
Now given a nonempty $C'\subseteq C_\omega$, put
\[
  C'_i:=\{x\in C_i: d(x,C')\le\delta_i\},
\]
which is a finite subset of $C_i$.
Every point of $C'$ lies within $\delta_i$ of some point of $C_i$, and that point then belongs
to $C'_i$, so $C'$ sits inside the $\delta_i$-neighbourhood of $C'_i$; conversely $C'_i$ sits
inside the $\delta_i$-neighbourhood of $C'$ by construction.
So $d_H(C'_i,C')\le\delta_i\to0$, and $C'_i$ is nonempty once $i\ge N$.
\end{proof}

\begin{theorem}[restating Theorem~\ref{thm:kleene}]\label{athm:kleene}
If $f$ is continuous for the Hausdorff metric, then $C_\omega=C^\bot_{\mathrm{cl}}$.
\end{theorem}

\begin{proof}
{\em $C_\omega$ is a fixed point of $F_{\mathrm{cl}}$.}
For $C_\omega\subseteq F_{\mathrm{cl}}(C_\omega)$: each
$C_{i+1}=F(C_i)\subseteq F_{\mathrm{cl}}(C_\omega)$ by isotonicity, so
$\bigcup_i C_i\subseteq F_{\mathrm{cl}}(C_\omega)$, and the latter is closed, which gives
$C_\omega\subseteq F_{\mathrm{cl}}(C_\omega)$.
For the other direction, since $C_\omega$ is closed it is enough to show $f(C')\in C_\omega$ for
every $C'\subseteq C_\omega$, and since $C_\omega$ is closed it then also contains
$\cl{F(C_\omega)}$.
The empty sub-answer is immediate, as $f(\emptyset)\in C_1\subseteq C_\omega$.
For a nonempty $C'$, Lemma \ref{alem:approx} gives finite $C'_i\subseteq C_i$ with $C'_i\to C'$;
then $f(C'_i)\in C_{i+1}\subseteq C_\omega$ by Lemma \ref{alem:finite}, and $f(C'_i)\to f(C')$
by continuity, so $f(C')\in C_\omega$ because $C_\omega$ is closed.

{\em $C_\omega$ is least.}
If $C\in\fix(F_{\mathrm{cl}})$ then $C_0=\emptyset\subseteq C$, and inductively $C_i\subseteq C$
implies $C_{i+1}=F(C_i)\subseteq\cl{F(C)}=C$.
So $\bigcup_i C_i\subseteq C$, and since $C$ is closed, $C_\omega\subseteq C$.
\end{proof}

\begin{remark}
Compactness of $[0,1]$ is used exactly once in the proof, and it is worth isolating where.
It guarantees that any closed $C'\subseteq C_\omega$ is a Hausdorff limit of finite
$C'_i\subseteq C_i$, so that $f(C')$ still gets caught in the closure even when $C'$ appears at
no finite stage.
Continuity is then what carries $f(C'_i)$ to $f(C')$.
\end{remark}

\begin{remark}
Continuity is used only in the limit step.
Without it the chain still climbs, but $C_\omega$ need not be a fixed point, and one carries on
transfinitely with $C_{\alpha+1}=F(C_\alpha)$ and
$C_\lambda=\cl{\bigcup_{\alpha<\lambda}C_\alpha}$ at limit ordinals.
Cardinality forces the chain to stabilise at some closure ordinal $\alpha^\star$ with
$C_{\alpha^\star}=C^\bot$, and continuity is precisely what pulls $\alpha^\star$ back down to
$\omega$.
Note also that finiteness is lost at $\omega$: $C_\omega$ may well be infinite, which is exactly
why the closure is needed there and nowhere earlier.
\end{remark}

\begin{remark}
The restriction to closed answers is not a convenience but a necessity, by Proposition
\ref{aprop:union}: on the powerset the chain's supremum is the plain union, which need not
contain the limits its members are heading for.
Everything else in Section \ref{app:lattice} is indifferent to the choice of lattice.
\end{remark}

\section{The Even-Handed Variant}\label{app:cstar}

\begin{proposition}[restating Proposition~\ref{prop:cstar}]\label{aprop:cstar}
Let $\SC=\{p\in[0,1]:f(\{p\})=p\}$ and $x=\{a\}\cup\SC$ (on $D_{\mathrm{cl}}$, $x=\cl{\{a\}\cup\SC}$).
Then $x\subseteq F(x)$.
Consequently the order interval $[x,[0,1]]$ of $D$ is $F$-invariant and carries a least fixed
point $C^\star$, the smallest self-consistent credal set containing $a$ and every
self-consistent singleton; and $C^\bot\subseteq C^\star$.
\end{proposition}

\begin{proof}
First, $a\in F(x)$ because $\emptyset\subseteq x$.
Next, for each $p\in\SC$ we have $\{p\}\subseteq x$, so $f(\{p\})=p\in F(x)$,
whence $\SC\subseteq F(x)$, and therefore $x=\{a\}\cup\SC\subseteq F(x)$.
On $D_{\mathrm{cl}}$ the same computation gives $x\subseteq F_{\mathrm{cl}}(x)$, since
$F_{\mathrm{cl}}(x)$ is closed.
So $x$ is a post-fixed point.
As $F$ is isotone, $y\supseteq x$ implies $F(y)\supseteq F(x)\supseteq x$, so $[x,[0,1]]$ is
mapped into itself; it is a complete lattice, and Knaster--Tarski gives it a least fixed point
$C^\star$.
Finally $C^\bot$ is contained in every fixed point of $F$ on all of $D$, so in particular in
$C^\star$.
\end{proof}

\section{The Interval Lattice}\label{app:interval}

Let $D'$ consist of the closed subintervals of $[0,1]$ together with $\emptyset$, ordered by
inclusion, with intersection as meet and convex hull of the union as join, so that $D'$ is a
complete lattice.
Write $g(p)=f([p,p])$ and $R(I)=\{f([s,t]):[s,t]\subseteq I\}$.
Because the join is a convex hull, the operator on $D'$ is
$F_{\mathrm{int}}(I)=[\,a\wedge\inf R(I),\;a\vee\sup R(I)\,]$, so only the two extremes of $R(I)$ ever
matter.

\begin{proposition}[Cutting the search down to one parameter]\label{aprop:onedim}
Let $f$ be continuous, so that $R(I)$ is compact and its extremes are attained.
\begin{enumerate}
\item[\textup{(C1)}] If $f([s,t])$ lies between $g(s)$ and $g(t)$ for all $s\le t$, then
$\inf R(I)=\min_{p\in I}g(p)$ and $\sup R(I)=\max_{p\in I}g(p)$.
\item[\textup{(C2)}] If $f([s',t'])$ lies between $f([s,t'])$ and $f([s',t])$ whenever
$s\le s'\le t'\le t$, then $\inf R(I)$ and $\sup R(I)$ are attained on the two edges
$s=\min I$ and $t=\max I$.
\end{enumerate}
\end{proposition}

\begin{proof}
For (C1), if $[s,t]\subseteq I$ then $g(s)$ and $g(t)$ both lie in
$g(I)\subseteq[\min_I g,\max_I g]$, and the hypothesis puts $f([s,t])$ between them, hence in
that same interval.
The diagonal $s=t$ attains both ends, so the inclusion is an equality.

For (C2), write $I=[\ell,u]$.
Taking $s=\ell$ and $t=u$ in the hypothesis brackets each $f([s',t'])$ with $\ell<s'\le t'<u$
between the edge values $f([\ell,t'])$ and $f([s',u])$.
Points with $s'=\ell$ or $t'=u$ are themselves edge values, and the remaining points of the
parameter triangle, the diagonal in particular, are limits of interior ones, so by continuity
the extremes over $R(I)$ are attained on the two edges.
\end{proof}

\begin{remark}
Convexity of $f$ is not a substitute for either condition.
A convex function attains its maximum at an extreme point of the parameter triangle but its
minimum generically in the interior, so it locates only one of the two extremes.
It is also worth saying that bisecting $I$ at some interior point $c$ and looking only at
$[\ell,c]$ and $[c,u]$ does not work: an announcement straddling $c$ can react outside the
ranges of both halves, so the bisection quietly misses it.
\end{remark}

\begin{theorem}[restating Theorem~\ref{thm:interval}]\label{aprop:orbit}
Assume \textup{(C1)} and let $g$ be continuous and nondecreasing.
Then $F_{\mathrm{int}}([\ell,u])=[\,a\wedge g(\ell),\;a\vee g(u)\,]$, iterating from $\{a\}$ gives
$I_k=\conv\{a,g(a),\dots,g^k(a)\}$, and
\[
  C^\bot_{\mathrm{int}}=[\,a\wedge p^\star,\;a\vee p^\star\,],
\]
where $p^\star=\lim_k g^k(a)$ is the stable fixed point of $g$ in whose basin $a$ lies.
\end{theorem}

\begin{proof}
Under (C1) together with monotonicity, $R([\ell,u])=[g(\ell),g(u)]$, which gives the displayed
form of $F_{\mathrm{int}}$, and the two endpoints then evolve independently under $g$, clamped by the anchor.
Starting at $I_1=\{a\}$ and iterating produces the convex hull of the scalar orbit $x_0=a$,
$x_{k+1}=g(x_k)$, by induction.
Since $g$ is nondecreasing, that orbit is monotone: if $x_1\ge x_0$ then applying $g$ gives
$x_{k+1}\ge x_k$ for all $k$, and symmetrically in the decreasing case.
A bounded monotone sequence converges, and by continuity its limit $p^\star$ satisfies
$g(p^\star)=p^\star$.
So the hull of the orbit is the interval with endpoints $a$ and $p^\star$.
We may stop as soon as $x_{k+1}\in\conv\{x_0,\dots,x_k\}$, since by monotonicity no later
iterate leaves that hull.
It remains to check that the resulting interval is a fixed point and is least.
Taking the increasing case $a\le g(a)$ without loss of generality, $I_{k+1}=[a,x_k]$ by
induction, so the chain increases to $[a,p^\star)$ and its supremum in $D'$ is $[a,p^\star]$.
That is a fixed point, since
$F_{\mathrm{int}}([a,p^\star])=[a\wedge g(a),\,a\vee g(p^\star)]=[a,p^\star]$.
It is least because any fixed interval contains $a$, hence by induction each $I_k$, hence
$[a,p^\star)$, and being closed it contains $p^\star$ as well.
\end{proof}

\begin{remark}
The interval statement is about $D'$ and does not survive passage to the full lattice of closed
sets.
Under the same hypotheses, take $f(C)=g(\sup C)$ for nonempty $C$.
Then on the full lattice the canonical answer is $\{g^k(a):k\ge0\}\cup\{p^\star\}$, a countable
discrete set with the right two endpoints and nothing in between.
What fills the gaps in $D'$ is the join being a convex hull rather than a union.
\end{remark}

\begin{remark}
A monotone orbit cannot cross an unstable fixed point of $g$, so $C^\bot_{\mathrm{int}}$ stays inside the basin
containing $a$.
That is why this construction returns $C^\bot_{\mathrm{int}}$ rather than the even-handed $C^\star$.
It is also why reporting the interval up to the {\em nearest} stable fixed point is a different
thing altogether: basin and nearest disagree whenever the anchor lies between the unstable fixed
point and the midpoint of the two stable ones, and the interval to the nearest one need not even
be self-consistent, since announcing it can drive the querier over to the other equilibrium.
\end{remark}

\section{Beyond Binary Events}\label{app:beyond}

Let $\mathcal X$ be an outcome space, let $K$ be a compact metric space of answers, either the
space $\mathcal P(\mathcal X)$ of laws under a suitable topology or the closure of the range of
a statistic, and let $f\colon D_K\to K$ be the reaction, where $D_K$ is the set of closed
subsets of $K$.
Nothing in Section \ref{app:lattice} used more than this.
The proofs of Lemma \ref{alem:complete}, Theorems \ref{athm:exist} and \ref{athm:consistent},
and Propositions \ref{aprop:collapse}, \ref{aprop:nontrivial} and \ref{aprop:cstar} go through
word for word with $[0,1]$ replaced by $K$ and $P(B\mid A,C)$ replaced by
$\law(X\mid A,C)$ or $Y(\law(X\mid A,C))$.

\subsection{The Affine Case}

\begin{proposition}[Orbit hull]\label{aprop:hull}
Let $K$ be a compact convex subset of $\mathbb R^d$, let $D_K^{\mathrm{cvx}}$ be the complete
lattice of closed convex subsets of $K$, with intersection as meet and closed convex hull of the
union as join, and let $f(C)=T(\bary\,C)$ for an affine $T:K\to K$.
Then $F_{\mathrm{cvx}}(C)=\cl{\conv}(T(C)\cup\{a\})$, where $a$ is the anchor, and
\[
  C^\bot_{\mathrm{cvx}}=\cl{\conv}\{\,T^n a:n\ge0\,\}.
\]
\end{proposition}

\begin{proof}
The barycentres of the closed convex subsets of $C$ fill out $C$, so
$\{f(C'):C'\subseteq C\}=T(C)$; adjoining the reaction to the empty sub-answer and taking the
closed convex hull gives the stated $F_{\mathrm{cvx}}$.
The set on the right is closed, convex, contains $a$, and is mapped into itself by $T$, since
affine maps commute with convex combinations and the orbit is $T$-invariant, so it is a fixed
point.
Conversely, any fixed point contains $a$ and is closed under both $T$ and closed convex hulls,
hence contains every $T^na$ and their closed convex hull.
So the displayed set is the least one.
\end{proof}

\begin{lemma}[Absorption]\label{alem:absorb}
If $T^Na\in\conv\{a,\dots,T^{N-1}a\}$ for some $N$, then the same holds for every later iterate,
and $C^\bot_{\mathrm{cvx}}=\conv\{a,\dots,T^{N-1}a\}$ is a polytope.
\end{lemma}

\begin{proof}
Affineness commutes with convex combinations.
If $T^Na=\sum_i\lambda_iT^ia$ with $\lambda_i\ge0$ summing to one, then
$T^{N+1}a=\sum_i\lambda_iT^{i+1}a\in\conv\{Ta,\dots,T^Na\}\subseteq\conv\{a,\dots,T^{N-1}a\}$,
so no later iterate is extreme; induction finishes the argument.
\end{proof}

Absorption happens at some finite $N$ exactly when the limit lies in the relative interior of
the hull of the orbit, which is to say when the orbit winds around it.
That is the complex-eigenvalue and negative-eigenvalue case.
A purely positive real spectrum never encircles the limit: every iterate stays extreme, and
$C^\bot$ then has infinitely many extreme points accumulating at the stationary point, so it is
convex and closed but not a polytope.
When the orbit is collinear the hull degenerates to the segment from anchor to equilibrium, which
is the exact analogue of Proposition \ref{aprop:orbit}.

\subsection{General Statistics}

\begin{proposition}[restating Proposition~\ref{prop:general}]\label{aprop:general}
Let $Y:\mathcal P(\mathcal X)\to\mathbb R^d$ and $K=\cl{Y(\mathcal P(\mathcal X))}$.
Suppose $C\mapsto\law(X\mid A,C)$ is continuous and $Y$ is bounded and continuous, so that
$f=Y\circ\law(X\mid A,\cdot)$ is continuous and $K$ is compact.
Then for each of the lattices $E$ of \textup{(a)} closed subsets and \textup{(b)} closed convex
subsets of $K$, the operator $F_E$ is $\omega$-Scott-continuous and $C^\bot_E$ is reached at
stage $\omega$.
\end{proposition}

\begin{proof}
Boundedness makes $K$ closed and bounded, hence compact by Heine--Borel, and continuity of the
composite is assumed.
Case (a) is Theorem \ref{athm:kleene} in its compact-metric form, whose proof used only
compactness of the ambient space and Hausdorff continuity of $f$.
For (b) we only need the approximation lemma to survive.
Given an increasing chain $C_i\uparrow C_\infty$ in $E$ and any $C'\subseteq C_\infty$ in $E$,
the truncations $C'\cap C_i$ stay in $E$, since an intersection of closed convex sets is closed
convex, and they converge to $C'$ in Hausdorff distance by the argument of Lemma
\ref{alem:approx}.
Then $f(C'\cap C_i)\in F_E(C_i)$ and $f(C'\cap C_i)\to f(C')$, so $f(C')$ lies in the supremum
of the $F_E(C_i)$; applying the closure of $E$ gives $F_E(C_\infty)\subseteq\sup_iF_E(C_i)$, and
the reverse inclusion is just isotonicity.
\end{proof}

\begin{remark}
Existence of $C^\bot$, $C^\star$ and $C^\top$, along with self-consistency, nonemptiness and
collapse, needs no assumption on $Y$ whatsoever.
Each of these is a complete lattice and its operator $F_E$ is isotone by construction, and that
is all Knaster--Tarski consumes.
Boundedness and continuity are needed only to {\em reach} $C^\bot$ by iteration at stage $\omega$.
If $Y$ is unbounded we adjoin the ambient space as a top element to keep the lattice complete; existence survives, but increasing chains need no longer converge and
reaching $C^\bot$ may take transfinite iteration, as it also may when $Y$ is merely
discontinuous.
Only case (a) keeps the finite-stage property.
The stages in (b) are genuinely continuum-valued, though for affine reactions they are at least
polytopes with finitely many vertices.
\end{remark}

\section{Finite Approximation}\label{app:finite}

Throughout this section the ambient space is the cube $[0,1]^d$, answers are convex subsets of
it, $h_C(z)=\sup_{y\in C}\langle z,y\rangle$ is the support function, and $\mathcal M_k$ is the
family generated by the slices $H_{z,c}=\{x\in[0,1]^d:\langle z,x\rangle\le c\}$ with
$z\in Z_k=\{-k,\dots,k\}^d\setminus\{0\}$ and $c\in \Gamma_k=\tfrac1k\mathbb Z\cap[-dk,dk]$.
We use $d_H(K,L)=\sup_{u\in S^{d-1}}|h_K(u)-h_L(u)|$ for nonempty compact convex $K,L$
\citep{schneider2014}.
Note that axis-parallel boxes are the special case of $\mathcal M_k$ in which only the $2d$
coordinate normals are used, so nothing below needs to treat them separately.

\begin{lemma}[restating Lemma~\ref{lem:moore}]\label{alem:moore}
$\mathcal M_k$ is finite, consists of polytopes, contains $\emptyset$ and $[0,1]^d$, and is
closed under intersection.
Its closure operator satisfies
\[
  \mathcal M_k(C)=\bigcap_{z\in Z_k}H_{z,c_z},\quad
  c_z=\min\{c\in \Gamma_k:c\ge h_C(z)\}.
\]
For $d=1$ it contains every $\mathcal M^{(n)}$ with $n\le k$.
\end{lemma}

\begin{proof}
$\mathcal H_k$ is finite, so the set of intersections of its subfamilies is finite; each member
is a bounded intersection of finitely many halfspaces, hence a polytope.
The empty subfamily gives $[0,1]^d$, and $z=e_1$ with $c=-dk$ gives $\emptyset$.
Closure under intersection holds because an intersection of intersections of subfamilies is the
intersection of their union.
For the formula, $C\subseteq\bigcap\mathcal S$ iff $C\subseteq H$ for every $H\in\mathcal S$, so
the smallest member containing $C$ is the intersection of all generators containing $C$; and
$C\subseteq H_{z,c}$ iff $c\ge h_C(z)$, so for each $z$ only the least admissible $c$ matters.
That minimum exists because $h_C(z)\le\|z\|_1\le dk$ on the cube.
For $d=1$ and $n\le k$, taking $z=\pm1$ and $c\in\tfrac1k\mathbb Z\supseteq\tfrac1n\mathbb Z$
produces every interval with endpoints in $\tfrac1n\mathbb Z$.
\end{proof}

\begin{remark}
For $d=1$ the family is strictly finer than $\mathcal M^{(k)}$: taking $z=2$ and $c=1/k$ gives
the halfline $x\le 1/(2k)$, whose endpoint is off the grid $\tfrac1k\mathbb Z$.
Only the inclusion $\mathcal M^{(n)}\subseteq\mathcal M_k$ is used above, and being finer only
helps.
\end{remark}

\begin{theorem}[restating Theorem~\ref{thm:approx}]\label{athm:approx}
For every convex $C\subseteq[0,1]^d$ we have $d_H(C,\mathcal M_k(C))\le 3d/k$.
\end{theorem}

\begin{proof}
If $C=\emptyset$ both sides are empty, and replacing $C$ by $\cl C$ changes neither $h_C$ nor
the hull, so assume $C$ nonempty and closed.
Put $\delta=\sqrt d/k$ and $\eta=\sqrt d/(2k)$.

{\em The normals form a $\delta$-net of the sphere.}
Given $v\in S^{d-1}$, let $z$ be the componentwise nearest integer to $kv$; then $|z_i|\le k$,
so $z\in Z_k$ once $z\neq0$. 
The latter holds because $\|kv-z\|_2\le\sqrt d/2$ with $k>\sqrt d/2$ forces
$\|z\|_2 + \|kv-z\|_2\ge \|z+kv-z\|_2 = k$, hence
$\|z\|_2\ge k-\sqrt d/2>0$.
Put $w=z/k$ and $u=z/\|z\|_2=w/\|w\|_2$.
Then $\|w-v\|_2\le\eta$ and $\bigl|\,\|w\|_2-1\bigr|\le\|w-v\|_2\le\eta$, so
\[
\begin{aligned}
  \|u-v\|_2&\le\Bigl\|\tfrac{w}{\|w\|_2}-w\Bigr\|_2+\|w-v\|_2\\
  &\le \bigl|1-\|w\|_2\bigr|+\eta\ \le\ 2\eta\ =\ \delta .
\end{aligned}
\]

{\em The offsets are $\delta$-fine and wide enough.}
Since $H_{z,c}=\{x\in[0,1]^d:\langle u,x\rangle\le c/\|z\|_2\}$, the available offsets in
normalised form are $\tfrac{1}{k\|z\|_2}\mathbb Z$, of spacing at most $1/k\le\delta$ because
$z$ is a nonzero integer vector, so $\|z\|_2\ge1$.
The range $|c|\le dk$ covers all normalised offsets in $[-\sqrt d,\sqrt d]$, since
$\|z\|_2\le k\sqrt d$ gives $dk/\|z\|_2\ge\sqrt d\ge|h_C(u)|$.
So the least available normalised offset $b_u$ above $h_C(u)$ satisfies
$h_C(u)\le b_u\le h_C(u)+\delta$.

{\em The estimate.}
(See Figure~\ref{fig:proof} for an illustration.)
Clearly $C\subseteq\mathcal M_k(C)$.
Let $x\in\mathcal M_k(C)$ with $t=\operatorname{dist}(x,C)>0$. As $C$ is closed, we can let $y\in C$ be nearest to $x$,
and put $v=(x-y)/t\in S^{d-1}$.
Since $C$ is convex,
the projection inequality $\langle v,y'-y\rangle\le0$ for $y'\in C$ gives
$h_C(v)=\langle v,y\rangle$, hence $\langle v,x\rangle=h_C(v)+t$.
Choose $u$ as above.
Using $\|x\|_2\le\sqrt d$ and that $h_C$ is Lipschitz with constant
$\sup_{y\in C}\|y\|_2\le\sqrt d$,
\[
\begin{aligned}
  h_C(v)+t-\delta\sqrt d&=\langle v,x\rangle-\delta\sqrt d\le\langle u,x\rangle\le b_u\\
  &\le h_C(u)+\delta\le h_C(v)+\delta\sqrt d+\delta .
\end{aligned}
\]
Therefore $t\le(2\sqrt d+1)\delta=(2d+\sqrt d)/k\le3d/k$, as $\sqrt d\le d$.
\end{proof}

\begin{figure}[t]\centering
\includegraphics[width=\linewidth]{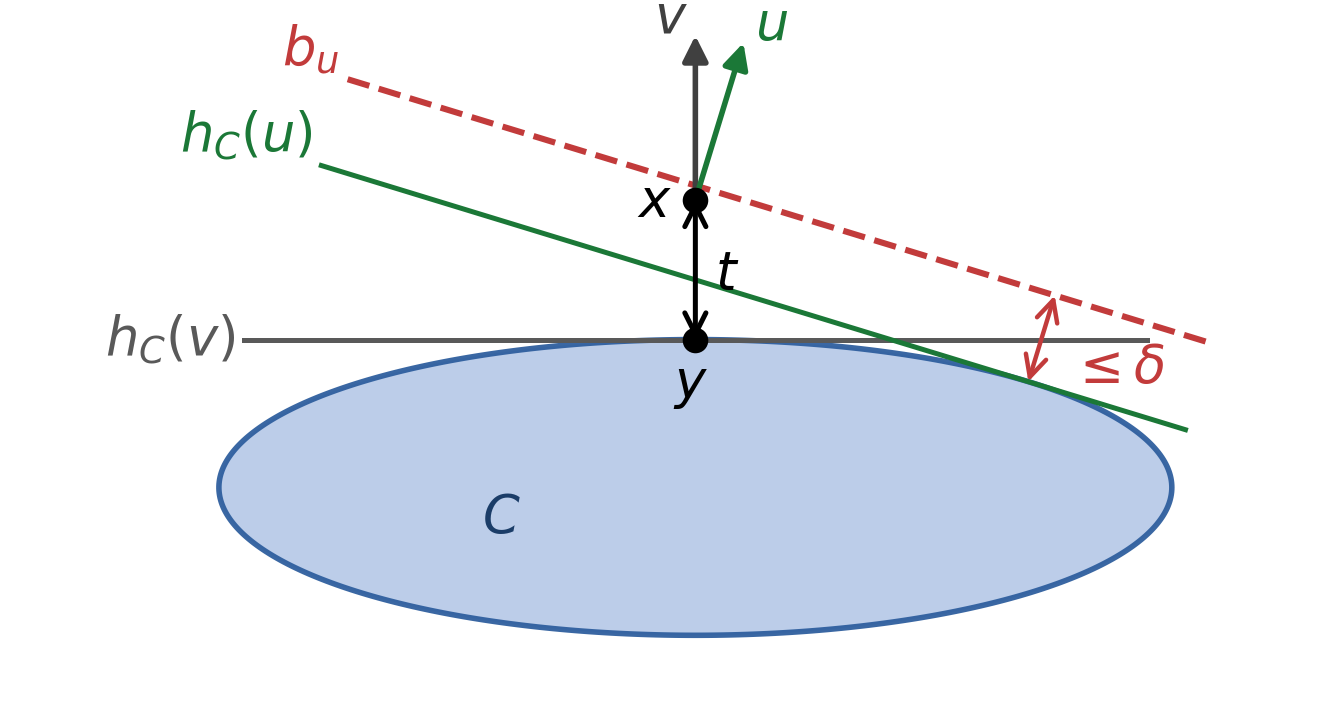}
\caption{The estimate in the proof of Theorem \ref{athm:approx}, drawn in $d=2$. The point
$x\in\mathcal M_k(C)$ lies at distance $t$ from $C$, with nearest point $y$ and
$v=(x-y)/t$; the horizontal line is the supporting hyperplane $\{\langle v,\cdot\rangle=h_C(v)\}$,
which touches $C$ at $y$. The lattice direction $u$ satisfies $\|u-v\|\le\delta$ (the angle is
drawn much larger than that, so the two can be told apart). Note that $x$ lies beyond the
supporting hyperplane $\{\langle u,\cdot\rangle=h_C(u)\}$, so the estimate must run through the
grid hyperplane $\{\langle u,\cdot\rangle=b_u\}$ that actually bounds $\mathcal M_k(C)$ in
direction $u$; that $b_u$ exceeds $h_C(u)$ by at most $\delta$ is what the offset grid
buys.}\label{fig:proof}
\end{figure}

\begin{lemma}[Canonical offsets and height]\label{alem:height}
For $M\in\mathcal M_k$ put $c_z(M)=\min\{c\in \Gamma_k:M\subseteq H_{z,c}\}$.
Then $M=\bigcap_{z}H_{z,c_z(M)}$, and for $M_1,M_2\in\mathcal M_k$ we have
$M_1\subseteq M_2$ if and only if $c_z(M_1)\le c_z(M_2)$ for all $z\in Z_k$.
Consequently every strictly increasing chain in $\mathcal M_k$ has at most
$|Z_k|(|\Gamma_k|-1)+1$ members, where $|Z_k|=(2k+1)^d-1$ and $|\Gamma_k|=2dk^2+1$.
\end{lemma}

\begin{proof}
$M$ is an intersection of generators, and replacing each by the tightest generator with the same
normal that still contains $M$ leaves $M$ unchanged, which gives the representation.
If $M_1\subseteq M_2$ then $M_1\subseteq H_{z,c_z(M_2)}$, so $c_z(M_1)\le c_z(M_2)$;
conversely, componentwise domination gives
$M_1=\bigcap_zH_{z,c_z(M_1)}\subseteq\bigcap_zH_{z,c_z(M_2)}=M_2$.
So inclusion is componentwise order on a vector of $|Z_k|$ coordinates, each ranging over the
$|\Gamma_k|$-element grid $\Gamma_k$, and along a strictly increasing chain at least one coordinate
strictly increases at each step, which can happen at most $|Z_k|(|\Gamma_k|-1)$ times.
Finally $\Gamma_k=\tfrac1k\mathbb Z\cap[-dk,dk]$ has $2dk^2+1$ elements.
\end{proof}

\begin{theorem}[restating Theorem~\ref{thm:finite}]\label{athm:finite}
$F_k=\mathcal M_k\circ F_{\mathrm{cvx}}$ is an isotone self-map of $\mathcal M_k$ with
$F_{\mathrm{cvx}}(C)\subseteq F_k(C)$ for all $C$, and $F_k(C)=\bigvee_{C'\subseteq C}G_k(C')$.
It has a least fixed point $C^\bot_k$, and a least fixed point $C^\star_k$ above
$\mathcal M_k(\{a\}\cup\SC)$; both are reached by iteration from below in at most
$|Z_k|(|\Gamma_k|-1)+1$ steps, and $C^\bot\subseteq C^\bot_k$, $C^\star\subseteq C^\star_k$.
\end{theorem}

\begin{proof}
$\mathcal M_k(\cdot)$ and $F_{\mathrm{cvx}}$ are both isotone and $\mathcal M_k(\cdot)$ is
extensive, so $F_k$
is an isotone self-map of $\mathcal M_k$ dominating $F_{\mathrm{cvx}}$.
A closure operator turns unions into joins, $\mathcal M_k(\bigcup_jS_j)=\bigvee_j\mathcal
M_k(S_j)$, and $F_{\mathrm{cvx}}(C)$ is the closed convex hull of $\{f(C'):C'\subseteq C\}$, whose
$\mathcal M_k$-hull is unchanged by taking convex hulls; so
$F_k(C)=\bigvee_{C'\subseteq C}\mathcal M_k(\{f(C')\})=\bigvee_{C'\subseteq C}G_k(C')$.
Existence of $C^\bot_k$ is Knaster--Tarski on the finite complete lattice $\mathcal M_k$; for
$C^\star_k$, note that $\mathcal M_k(\{a\}\cup\SC)\subseteq F_k(\mathcal M_k(\{a\}\cup\SC))$ by
the argument of Proposition \ref{aprop:cstar} together with extensiveness, so the order interval
above it is $F_k$-invariant and again a finite complete lattice.
The chain $C_0=\emptyset$, $C_{i+1}=F_k(C_i)$ is increasing, so by Lemma \ref{alem:height} it
stabilises within the stated number of steps, at the least fixed point by the usual induction.
For the last claim, if $F_k(C)\subseteq C$ then $F_{\mathrm{cvx}}(C)\subseteq F_k(C)\subseteq C$,
so every pre-fixed point of $F_k$ is one of $F_{\mathrm{cvx}}$, whence
$C^\bot_{\mathrm{cvx}}=\inf\{C:F_{\mathrm{cvx}}(C)\subseteq C\}\subseteq
\inf\{C:F_k(C)\subseteq C\}=C^\bot_k$, and likewise
above the seed for $C^\star$.
\end{proof}

\begin{proposition}[restating Proposition~\ref{prop:approxsc}]\label{aprop:approxsc}
Every $C\in\fix(F_k)$ satisfies $d_H(C,F_{\mathrm{cvx}}(C))\le3d/k$.
\end{proposition}

\begin{proof}
$F_{\mathrm{cvx}}(C)$ is convex, so Theorem \ref{athm:approx} gives
$d_H\bigl(F_{\mathrm{cvx}}(C),\mathcal M_k(F_{\mathrm{cvx}}(C))\bigr)\le3d/k$, and
$\mathcal M_k(F_{\mathrm{cvx}}(C))=F_k(C)=C$.
\end{proof}

\begin{remark}
The number of iterations is bounded, but each one still evaluates a supremum over sub-answers,
which is where the combinatorial cost now sits; this is the same issue as the $2^{|C_i|}$ in
Lemma \ref{alem:finite}, and structural assumptions such as the no-overshoot condition \eqref{eq:c1} cut it down.
Note also that the bound of Theorem \ref{athm:approx} is linear in $\delta$, not quadratic as in
Dudley's tangent-facet construction \citep{dudley1974}: for the segment $C=[-e_1,e_1]$ the hull
contains $te_2$ whenever $t\le\min\{|u_1|/u_2: u\in N,\ u_2>0\}$, which is $\Theta(\delta)$ for
generic segment directions.
The price is exactly that our normals are fixed in advance and so cannot be tangent to $C$,
which is what makes the family a closure system at all.
\end{remark}


\end{document}